\documentclass[letterpaper, 10 pt, conference]{ieeeconf}  %

\usepackage{times}

\usepackage{multicol}
\usepackage[bookmarks=true]{hyperref}

\usepackage{times}
\usepackage{algorithm}
\usepackage{algpseudocode}
\usepackage{amsmath}
\usepackage{amssymb}
\usepackage{graphicx}
\usepackage{array}
\usepackage{multirow}
\usepackage[table]{xcolor}
\usepackage{cite}
\usepackage{subcaption}
\usepackage{mysymbol}

\newtheorem{theorem}{Theorem}[section]

\newtheorem{prop}[theorem]{Proposition}

\newtheorem{problem}{Problem}

\newtheorem{rem}[theorem]{Remark}

\algtext*{EndIf}
\algtext*{EndFor}
\algtext*{EndWhile}
\algtext*{EndFunction}
\algtext*{EndProcedure}

\IEEEoverridecommandlockouts                              %

\title{\LARGE \bf
\textcolor{black}{Conformal Constraint Tightening for Chance-Constrained Motion Planning with Unknown Dynamics}
}

\author{Shubham Natraj$^{1}$, Bruno Sinopoli$^{2}$ and Yiannis Kantaros$^{1}$%
\thanks{*This work was funded by NSF awards CNS-2231257 and CCF-2403758.}%
\thanks{$^{1}$Shubham Natraj and Yiannis Kantaros are with the Department of Electrical and Systems Engineering, Washington University in St. Louis, 1 Brookings Dr, St. Louis, MO 63130, USA, {\tt\small \{n.shubham, ioannisk\}@wustl.edu}}%
\thanks{$^{2}$Bruno Sinopoli is with the School of Electrical, Computer and Energy Engineering, Arizona State University, 650 E. Tyler Mall, Tempe, AZ 85281, USA, {\tt\small bruno.sinopoli@asu.edu}}%
}

\begin{document}

\maketitle
\thispagestyle{empty}
\pagestyle{empty}

\begin{abstract}
Motion planning algorithms compute control sequences that drive autonomous robots to goal regions while avoiding unsafe states. Existing methods, from sampling-based planning to deep reinforcement learning, typically provide task-completion guarantees only with respect to a nominal model or simulator, which may be invalidated when the true dynamics are unknown or difficult to model accurately. This letter addresses this limitation for systems with unknown dynamics and an available approximate nominal model, contributing a planner-agnostic constraint-tightening procedure that equips existing planners with a probabilistic task-completion guarantee on the true system. 
We leverage conformal prediction to provide a probabilistic bound on the nominal-to-true trajectory deviation over a distribution of planning problems.
We tighten the planning constraints using that bound, 
and show that solving the tightened problem under the nominal model is a sufficient condition for solving the original problem on the true system with a prescribed probability.
We validate the theoretical guarantees empirically and demonstrate substantially improved task completion relative to nominal-model planning.

\end{abstract}

\section{Introduction}
\label{sec:intro}

Motion planning requires computing a sequence of control inputs that drives a robot from an initial state to a desired goal region while avoiding unsafe states. Numerous approaches have been proposed, including sampling-based planners \cite{lavalle2001randomized,li2016asymptotically,luo2021abstraction,perrault2025kino,duong2026ultrafast,wang2022rapidly}, model predictive control (MPC) methods \cite{fang2022model,luis2020online,chen2022interactive}, and deep-learning methods \cite{tsounis2020deepgait,malczyk2025semantically,safaoui2024safe,lu2023learning,rousseas2020optimal}. However, they typically assume access to an accurate model or simulator of the system. Consequently, their safety and task-completion guarantees may be invalidated by model mismatch. Here, we consider motion-planning problems in which the robot's true dynamics are unknown, while an \textit{approximate nominal model} is available. 
Our goal is to compute a control sequence whose execution on the true system reaches the goal region while avoiding unsafe states, with a prescribed probability over planning problems drawn from an unknown distribution. 
This problem is challenging since unmodeled effects, parameter uncertainty, and approximation errors in learned system dynamics can cause nominally safe trajectories to fail on the true system.

In this paper, we propose a planner-agnostic framework for motion planning under model mismatch.
First, we leverage conformal prediction (CP) to quantify the discrepancy between nominal %
and
true trajectories. Using a calibration dataset of nominal trajectories and their corresponding true trajectories obtained by executing the associated control sequences, CP provides a probabilistic bound on the worst-case nominal-to-true trajectory deviation over the planning horizon. %
This bound defines conformal balls around the states of %
any planned nominal trajectory that contain the corresponding true trajectory with a user-specified marginal probability. We then use the radius of these balls to \textcolor{black}{shrink} the free state space and goal region, yielding a \textcolor{black}{planning problem with tightened constraints}. We show that a control sequence computed for this tightened problem using the nominal model satisfies the original planning problem on the unknown true system with a prescribed probability. Thus, chance-constrained planning under unknown dynamics is reduced to deterministic planning under the nominal model which can be solved by any planning and control method.
We empirically evaluate our framework using a sampling-based planner,
demonstrating improved task-completion rates relative to uncertainty-agnostic baselines. We also empirically characterize the trade-off between conservativeness and task completion as the prescribed probability increases.

\textbf{Related Works:}
Several motion-planning and control methods address uncertainty in system models. We organize the most relevant work into the following categories.

\textit{Kinodynamic Planning:} Sampling-based planners for uncertain systems have been proposed in \cite{luders2010chance,summers2018distributionally,gracia2026provably,wu2022robust}. These methods generally assume a known stochastic system model together with structural assumptions on the uncertainty. For example, CC-RRT \cite{luders2010chance} plans for a \emph{known} discrete-time linear system subject to \emph{Gaussian} process noise and uncertain obstacles by enforcing a prescribed probability of feasibility at each time step. DR-RRT and WDR-X \cite{summers2018distributionally,gracia2026provably} relax distributional assumptions through ambiguity sets, but still propagate uncertainty using a \emph{known} stochastic dynamics model or model class. Robust-RRT \cite{wu2022robust} considers nonlinear systems subject to parametric uncertainty with known bounds, which enable reachability analysis during planning. If the modeling assumptions in these works are violated, the associated guarantees need not hold. In contrast, we assume access only to a nominal model and consider unknown true dynamics without requiring a parametric uncertainty model or a known analytical relationship between the nominal and true dynamics. We assume only that the nominal and true systems share the same state and control spaces.

\textit{Tracking Control:} Related work also considers tracking control under stochastic dynamics and model uncertainty \cite{micheli2022scenario,micheli2022data,michaux2025can,das2025spatiotemporal,mayne2011tube,pan2023distributionally,verginis2022kdf}. As in the kinodynamic-planning literature, many of these approaches assume known system models and impose structural assumptions on the disturbance process or model uncertainty. For example, \cite{pan2023distributionally} considers stochastic known linear systems whose noise mean belongs to an ambiguity set without considering safety constraints. %
The method in \cite{verginis2022kdf} %
combines geometric planning in a restricted free space obtained by inflating obstacles with user-chosen funnel bounds with a funnel-based feedback controller that guarantees tracking for a structured class of fully actuated systems. In contrast, our framework uses calibration data to determine the tightening of obstacle and goal sets required to provide a user-specified probabilistic task-completion guarantee. More broadly, these methods address tracking-control problems that typically assume that a reference trajectory is given and seek to bound the deviation of the executed trajectory from that reference. Thus, the guarantees provided by these approaches are fundamentally different from the task-completion guarantee considered in this work.

\textcolor{black}{\textit{Conformal Prediction:} CP has recently been applied to autonomy problems involving planning and control \cite{vlahakis2024conformal,chee2023uncertainty, chee2024uncertainty, srinivasan2026safety}. These works address problem settings and provide guarantees that differ from those considered here. For example, \cite{vlahakis2024conformal} uses CP for stochastic optimal control and trajectory tracking in known linear systems subject to disturbances drawn from an unknown distribution. Similarly, \cite{chee2023uncertainty, chee2024uncertainty} use weighted CP to quantify uncertainty in learned dynamics models within learning-based MPC for trajectory-tracking tasks. In contrast, we formulate CP directly at the level of reach-avoid motion-planning problems for systems with unknown true dynamics and an available nominal model.}

The work closest to ours is \cite{srinivasan2026safety}, which also uses CP to provide probabilistic guarantees for planning and control under model uncertainty. Our framework, however, applies CP over a distribution of planning problems that may vary in obstacle configurations, initial states, goal regions, and horizons, yielding a task-completion guarantee for an unseen problem drawn from this distribution. This is particularly relevant  in robotic applications when a planner is repeatedly invoked across new environments and missions. In contrast, \cite{srinivasan2026safety} applies CP to local model-prediction errors at state-control pairs for a given planning task and propagates the resulting uncertainty through an SLS-based MPC controller. Consequently, its guarantee concerns the closed-loop execution of that task and is not formulated over a distribution of planning problems whose environments and mission specifications may vary. Our framework is also planner-agnostic: the resulting tightened problem can be solved by any planner capable of enforcing the tightened constraints, whereas the guarantee in \cite{srinivasan2026safety} is tied to its SLS-based MPC construction.

\textit{Learning-based Approaches:} Finally, learning-based techniques have also addressed planning and control problems under unknown dynamics or model mismatch. Examples include joint model and policy learning \cite{lu2023learning, zhang2025motion}, active model learning for uncertainty-aware MPC \cite{saviolo2023active}, residual reinforcement learning \cite{johannink2018residual}, and sim-to-real transfer through dynamics randomization \cite{peng2018sim}. These methods address the broad challenge of acting under imperfect or unknown models, but they typically target empirical robustness and improved transfer performance without providing task-completion guarantees for the underlying true system.

\textbf{Contributions:} The main contributions of this paper are:
(1) We develop a CP-based uncertainty-quantification framework for systems with unknown true dynamics that uses only an approximate nominal model and calibration data collected from the true system to provide a probabilistic bound on trajectory-level nominal-to-true deviation over a distribution of planning problems.
(2) We use this bound to define a planning problem with tightened constraints and prove that every control sequence solving it for the nominal model also solves the original problem on the true system with a prescribed probability.
(3) We empirically validate these guarantees and demonstrate substantially improved task-completion rates compared with nominal-model planning.

\section{Problem Formulation}\label{sec:problem}

Consider a robot with discrete-time dynamics
\begin{equation}\label{eq:true-dyn}
    \bbx_{t+1} = f(\bbx_t,\bbu_t),
\end{equation}
where $\bbx_t \in \ccalX \subset \mathbb{R}^n$ denotes the robot state at time $t$, $\ccalX$ is the state space, and $\bbu_t \in \ccalU$ is the control input. The true system dynamics $f$ are assumed to be unknown. Instead, we have access to a deterministic nominal model
\begin{equation}\label{eq:proxy-dyn}
\hat{\bbx}_{t+1} = \hat{f}(\hat{\bbx}_t,\bbu_t),
\end{equation}
where $\hat{\bbx}_t \in \ccalX$ denotes the nominal state. The nominal model is subject to the same state and input constraints as  \eqref{eq:true-dyn}. While we focus here on deterministic true dynamics \eqref{eq:true-dyn}, extensions to stochastic dynamics %
are discussed in Remark~\ref{rem:stochastic-dynamics}. 

The objective is to compute a sequence of control inputs $\bbu_t$ driving
the robot from an initial state $\bbx_{\mathrm{init}}$ to a goal region $\ccalX_{\mathrm{goal}}\subseteq\ccalX$ within a finite horizon $H$, while avoiding unsafe states collected in $\ccalX_{\mathrm{obs}}\subset\ccalX$. We denote the corresponding safe state space by $\ccalX_{\mathrm{free}}:=\ccalX\setminus\ccalX_{\mathrm{obs}}$. Such planning problems are represented by tuples $\ccalM = (\ccalX,\ccalX_{\mathrm{free}},\bbx_{\mathrm{init}},\ccalX_{\mathrm{goal}},H)$. We assume that planning problems are drawn independently from an unknown distribution $\ccalD_{\ccalM}$, from which i.i.d. samples can be obtained.

The problem addressed in this paper is stated as follows:
\begin{problem}[Chance-Constrained Safe Planning]\label{pr:fixed-alpha}
Given a failure level $\alpha \in (0,1)$, a nominal model $\hat{f}$ in \eqref{eq:proxy-dyn}, 
and a planning problem $\ccalM \sim\ccalD_{\ccalM}$,
compute a control sequence $\bbu_{0:T-1} := (\bbu_0,\bbu_1,\dots,\bbu_{T-1})$ with $T \leq H$ such that the trajectory followed by the true system in \eqref{eq:true-dyn} satisfies the assigned task with probability at least $1-\alpha$, i.e.,
\begin{align}\label{eq:fixed-alpha-goal} %
    \mathbb{P}\Bigl(
    \wedge_{t=0}^{T}(\bbx_t \in \ccalX_{\mathrm{free}}) \ 
    \wedge
    (\bbx_T \in \ccalX_{\mathrm{goal}})
    \Bigr)
    \geq 1-\alpha.
\end{align}
\end{problem}

We emphasize that the probabilistic guarantee in \eqref{eq:fixed-alpha-goal} is not a per-task guarantee. Instead, the probability is taken marginally with respect to a planning problem $\ccalM\sim\ccalD_{\ccalM}$.

\section{Probabilistically Safe Planning}\label{sec:cp-cert}

This section presents our solution to Problem~\ref{pr:fixed-alpha}. First, we use Conformal Prediction (CP) to quantify the discrepancy between nominal trajectories generated by \eqref{eq:proxy-dyn} and true trajectories generated by \eqref{eq:true-dyn}; see Sections~\ref{sec:cp-background}-\ref{sec:cp-calibration}. Next, given a new planning problem $\ccalM\sim\ccalD_{\ccalM}$, we use the resulting threshold to construct a tightened planning problem $\ccalM_\alpha$ and show that any control sequence $\bbu_{0:T-1}$ that solves $\ccalM_\alpha$ for the nominal model also solves Problem~\ref{pr:fixed-alpha}; see Section~\ref{sec:certified-planning}.

\subsection{Background on Conformal Prediction}\label{sec:cp-background}

CP provides distribution-free, finite-sample valid uncertainty quantification for the predictions 
of any fixed model \cite{angelopoulos2023conformal,shafer2008tutorial}. 
Formally, consider a predictive model $\sigma : \mathbb{X} \to \mathbb{Y}$ that maps an input $X \in \mathbb{X}$ to a prediction $\sigma(X)\in \mathbb{Y}$ of a ground truth $Y \in \mathbb{Y}$. CP does not make any assumptions on the underlying model or data distribution. Instead, it relies on a user-defined \textit{nonconformity score} (NCS) function $s: \mathbb{X} \times \mathbb{Y} \to \mathbb{R}$ that quantifies the error or mismatch between the prediction $\sigma(X)$ and the true outcome $Y$. 

Let $\ccalS_{\text{cal}} = \{(X_i, Y_i)\}_{i=1}^{N_{\mathrm{cal}}}$ be a calibration dataset consisting of $N_{\mathrm{cal}}$ i.i.d. samples drawn from a distribution $\mathcal{D}$. We compute the calibration scores $R_i = s(X_i, Y_i)$ for all $i \in \{1, \dots, N_{\mathrm{cal}}\}$. Given an unseen test point $X \sim \mathcal{D}$ with an unknown ground truth $Y$, its corresponding unknown score is $R_{\mathrm{test}} = s(X,Y)$. Given a target failure probability $\alpha\in(0,1)$, CP computes a threshold $\hat q_\alpha$ for which the following marginal coverage guarantee holds:
\begin{equation}\label{eq:cp-validity-general}
    \mathbb{P}\bigl(R_{\mathrm{test}} \le \hat{q}_{\alpha}\bigr) \ge 1-\alpha,
\end{equation}
where $\hat{q}_{\alpha}$ is the $(1-\alpha)$-th empirical quantile of the set $\{R_1,\dots,R_{N_{\mathrm{cal}}},\infty\}$. Note that the probability in \eqref{eq:cp-validity-general} is marginal, holding over the combined randomness of the calibration data and the new test sample drawn from $\mathcal D$; \textcolor{black}{all probabilities considered hereon are of this form}.

\subsection{Quantifying Trajectory-level Model Mismatch with CP}\label{sec:cp-calibration}

Our approach leverages CP to quantify the uncertainty of the nominal model $\hat f$ in \eqref{eq:proxy-dyn}, i.e., the discrepancy between nominal trajectories generated by \eqref{eq:proxy-dyn} and the corresponding true trajectories generated by \eqref{eq:true-dyn}. Applying CP in this setting requires two components: (i) a calibration dataset comprising planning problems and true trajectories collected from the unknown system, and (ii) an NCS that measures the prediction error of the nominal model on each planning problem. %
Using these components, CP computes a threshold $\hat q_\alpha$ that bounds the nominal-to-true trajectory deviation with user-specified probability.

\textbf{Calibration Dataset:}
Consider a planning problem $\ccalM\sim\ccalD_{\ccalM}$. Let $\ccalU(\ccalM)$ denote the set of all finite control sequences $\bbu_{0:T-1}:=(\bbu_0,\ldots,\bbu_{T-1})$, with $T\leq H$, whose nominal trajectories generated by \eqref{eq:proxy-dyn} solve $\ccalM$. That is, each $\bbu_{0:T-1}\in\ccalU(\ccalM)$ generates nominal states $\hat{\bbx}_{0:T}:=(\hat{\bbx}_0,\ldots,\hat{\bbx}_T)$, \textcolor{black}{where $\hat{\bbx}_0=\bbx_{\mathrm{init}}$} , satisfying $\hat{\bbx}_t \in \ccalX_{\mathrm{free}}$ for all $t \in \{0,\dots,T\}$ and $\hat{\bbx}_T \in \ccalX_{\mathrm{goal}}$. During calibration, we assume physical access to the true system \eqref{eq:true-dyn}, allowing  nominally feasible control sequences to be executed and their resulting true trajectories $\bbx_{0:T}:=(\bbx_0,\ldots,\bbx_T)$ to be collected, \textcolor{black}{with $\bbx_0=\bbx_{\mathrm{init}}$}. The calibration dataset is 
$\ccalS_{\mathrm{cal}} =\{(\ccalM_i,\ccalY_i)\}_{i=1}^{N_{\mathrm{cal}}},$
where $\ccalM_i$ are drawn i.i.d. from $\ccalD_{\ccalM}$ and $\ccalY_i$ denotes the collection of true trajectories associated with selected nominally feasible control sequences from $\ccalU(\ccalM_i)$. In practice, since executing every control sequence in $\ccalU(\ccalM_i)$ is generally intractable, $\ccalY_i$ is constructed using a finite subset of nominally feasible control sequences; see Rem.~\ref{rem:practical-score}.

\textbf{Nonconformity Score (NCS):}
For a planning problem $\ccalM$, we assign to the nominal model $\hat f$ the following score, defined as the worst-case \textcolor{black}{nominal-to-true trajectory deviation} over all nominally feasible control sequences:
\begin{equation}\label{eq:ideal-score}
    R(\ccalM)
    :=
    \sup_{\bbu_{0:T-1}\in\ccalU(\ccalM)} \max_{t \in \{0,\dots,T\}} \|\bbx_t-\hat{\bbx}_t\|_2.
\end{equation}

\textbf{Calibration Step:}
For each calibration planning problem $\ccalM_i$, we compute the corresponding score $R_i:=R(\ccalM_i)$. Given the calibration scores $\{R_i\}_{i=1}^{N_{\mathrm{cal}}}$ and a target failure probability $\alpha\in(0,1)$, we compute the conformal threshold $\hat q_\alpha$ as the $(1-\alpha)$-th empirical quantile of $\{R_1,\dots,R_{N_{\mathrm{cal}}},\infty\}$.

\textbf{Nominal-to-True Model Error Bound:}
Given the test-time planning problem $\ccalM\sim\ccalD_{\ccalM}$, whose true trajectories generated by \eqref{eq:true-dyn} are unavailable, the corresponding score $R(\ccalM)$ in \eqref{eq:ideal-score} is unknown. Applying CP as in \eqref{eq:cp-validity-general} yields
\begin{equation}\label{eq:tube-guarantee}
    \mathbb{P}\!\left(
    \sup_{\bbu_{0:T-1}\in\ccalU(\ccalM)}
    \max_{t \in \{0,\dots,T\}}
    \|\bbx_t-\hat{\bbx}_t\|_2
    \le \hat q_\alpha
    \right)
    \ge 1-\alpha.
\end{equation}
Geometrically, \eqref{eq:tube-guarantee} means that, with probability at least $1-\alpha$, for every nominally feasible control sequence $\bbu_{0:T-1}$, the corresponding true states lie inside the sequence of radius-$\hat q_\alpha$ balls centered at the nominal states. This motivates the following sufficient condition for satisfying the chance constraint in Problem~\ref{pr:fixed-alpha}.

\begin{prop}[Planning Guarantee]\label{prop:cp-certification}
Let $\ccalM \sim \ccalD_{\ccalM}$ be the planning problem in \eqref{eq:tube-guarantee}, and let $\bbu_{0:T-1}\in\ccalU(\ccalM)$ be \textit{any} nominally feasible control sequence yielding a nominal trajectory $\hat{\bbx}_{0:T}$. Let $\mathcal B(\bbz,r)$ denote the closed Euclidean ball of radius $r$ centered at $\bbz$. If the sequence of radius-$\hat q_\alpha$ balls centered at the nominal states is contained in the free space and its terminal ball is contained in the goal region, i.e.,
\begin{equation}\label{eq:alpha-certified}
    \mathcal B(\hat{\bbx}_t,\hat q_\alpha) \subseteq \ccalX_{\mathrm{free}},
    \ \forall t \in \{0,\dots,T\},
    ~\mathcal B(\hat{\bbx}_T,\hat q_\alpha) \subseteq \ccalX_{\mathrm{goal}},
\end{equation}
then 
\textcolor{black}{$\bbu_{0:T-1}$}
solves Problem~\ref{pr:fixed-alpha}.

\end{prop}

\begin{proof} By \eqref{eq:tube-guarantee},  with probability at least $1-\alpha$, we have that: $\|\bbx_t-\hat{\bbx}_t\|_2 \le \hat q_\alpha,
\qquad
\forall t \in \{0,\dots,T\},$
for any nominally feasible control sequence $\bbu_{0:T-1}\in\ccalU(\ccalM)$. 
Equivalently, $\bbx_t \in \mathcal B(\hat{\bbx}_t,\hat q_\alpha)$ for all $t\in\{0,\dots,T\}$ with probability at least $1-\alpha$. If \eqref{eq:alpha-certified} holds, then these balls are contained in $\ccalX_{\mathrm{free}}$ for all $t\in\{0,\dots,T\}$, and the terminal ball is contained in $\ccalX_{\mathrm{goal}}$. Thus, the true trajectory remains in $\ccalX_{\mathrm{free}}$ at every time and ends in $\ccalX_{\mathrm{goal}}$ with probability 
\textcolor{black}{at least $1-\alpha$, i.e., \eqref{eq:fixed-alpha-goal} holds.}
\end{proof}

\begin{rem}[Practical Approximation of the NCS]\label{rem:practical-score}
Computing the NCS in \eqref{eq:ideal-score} exactly requires optimizing over all nominally feasible control sequences in $\ccalU(\ccalM)$, which is generally intractable. 
In line with related work \cite{sundarsingh2026safe, mei2026perceive},
we approximate it using a finite set of control sequences. 
\end{rem}

\begin{rem}[Extension to Stochastic Dynamics]\label{rem:stochastic-dynamics}
The above CP analysis can also be applied to stochastic true system dynamics of the form $\bbx_{t+1}=f(\bbx_t,\bbu_t,\bbw_t)$, where the disturbance $\bbw_t$ takes values in a bounded set $\ccalW$. In this case, the score in \eqref{eq:ideal-score} can be replaced by a worst-case score over both nominally feasible control sequences and admissible disturbance realizations, i.e., $R_{\mathrm{stoch}}(\ccalM)
    := \sup_{\substack{\bbw_{0:T-1}\in\ccalW^{T}
    }}
    R(\ccalM)$,
where $\ccalW^{T}$ denotes the set of length-$T$ disturbance sequences for the same horizon as $\bbu_{0:T-1}$ and $ R(\ccalM)$ is defined in \eqref{eq:ideal-score}.
Applying CP using $R_{\mathrm{stoch}}$ yields the same form of marginal guarantee as \eqref{eq:tube-guarantee}. This extension can be conservative, especially when $\ccalW$ is large. 
As discussed in Rem. \ref{rem:practical-score} the NCS can be approximated by considering a finite set of feasible control sequences evaluated on the true stochastic system. 
We include a stochastic case study in Section~\ref{sec:experiments}; reducing the conservativeness of this construction is left for future work.
\end{rem}

\begin{figure}[t]
    \centering
    \includegraphics[width=\linewidth]{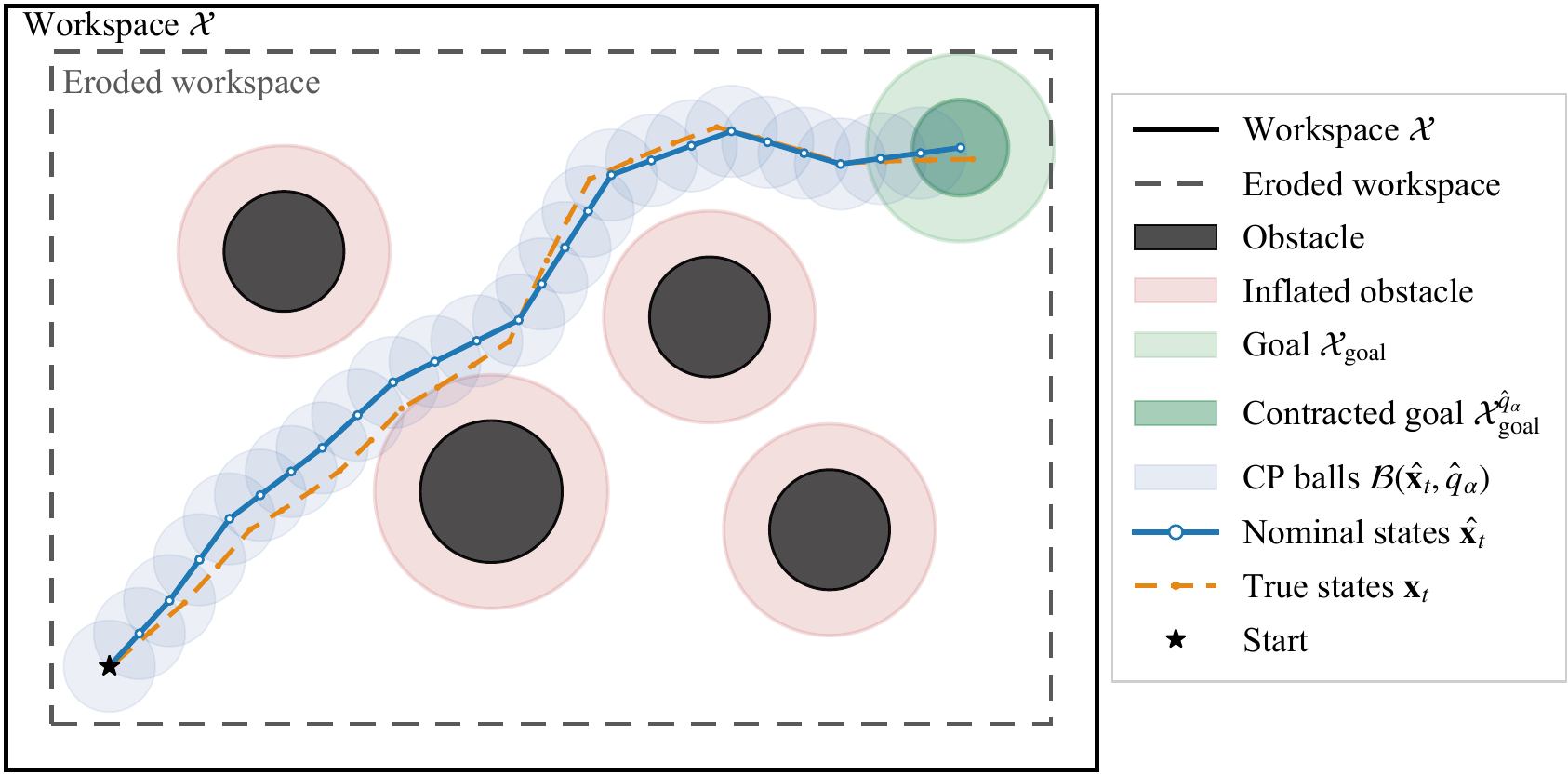}
    \caption{Illustration of the proposed CP-based planning framework. The conformal threshold $\hat q_\alpha$ induces a planning problem with tightened constraints such that any control sequence solving it inherits the reach-avoid guarantee of Prop.~\ref{prop:planner-certification}.}\vspace{-0.5cm}
    \label{fig:method}
\end{figure}

\subsection{Probabilistically Safe Planning}\label{sec:certified-planning}

Given the planning problem $\ccalM$ and the conformal threshold $\hat q_\alpha$, we define a planning problem with tightened constraints, $\ccalM_\alpha := (\ccalX,\ccalX_{\mathrm{free}}^{\hat q_\alpha},\bbx_{\mathrm{init}},\ccalX_{\mathrm{goal}}^{\hat q_\alpha},H)$, where
\begin{equation}\label{eq:tightened-sets}
\begin{aligned}
    \ccalX_{\mathrm{free}}^{\hat q_\alpha}
    &:= \{\bbx\in\ccalX:\mathcal B(\bbx,\hat q_\alpha)\subseteq\ccalX_{\mathrm{free}}\},\\
    \ccalX_{\mathrm{goal}}^{\hat q_\alpha}
    &:= \{\bbx\in\ccalX:\mathcal B(\bbx,\hat q_\alpha)\subseteq\ccalX_{\mathrm{goal}}\}.
\end{aligned}
\end{equation}
Intuitively, $\ccalX_{\mathrm{free}}^{\hat q_\alpha}$ and $\ccalX_{\mathrm{goal}}^{\hat q_\alpha}$ 
\textcolor{black}{are obtained by shrinking the free and goal sets by $\hat q_\alpha$. In the planar navigation example of Fig.~\ref{fig:method}, this is equivalent to inflating the obstacles, eroding the workspace boundary, and contracting the goal region.}
Solving $\ccalM_\alpha$ for the nominal model \eqref{eq:proxy-dyn} \textcolor{black}{thus} requires \textcolor{black}{the} nominal state to eventually reach $\ccalX_{\mathrm{goal}}^{\hat q_\alpha}$ while remaining in $\ccalX_{\mathrm{free}}^{\hat q_\alpha}$; this can be addressed by any suitable planner, including sampling-based motion planners such as RRT \cite{lavalle1998rapidly,lavalle2001randomized} and SST \cite{li2016asymptotically}.
Next, we show that solving $\ccalM_\alpha$ yields a solution to Problem \ref{pr:fixed-alpha}.
\begin{prop}[Planner Guarantee]\label{prop:planner-certification}
\textcolor{black}{Every} control sequence that solves the tightened problem $\ccalM_\alpha$ for the nominal model \eqref{eq:proxy-dyn} also solves the original problem $\ccalM$ for the true system \eqref{eq:true-dyn} with probability at least $1-\alpha$.

\end{prop}

\begin{proof}
Let $\bbu_{0:T-1}$ be a control sequence returned by a planner that solves $\ccalM_\alpha$ for the nominal model, and let $\hat{\bbx}_{0:T}$ be the corresponding nominal trajectory. 
\textcolor{black}{By construction of $\ccalM_\alpha$, we have $\hat{\bbx}_t\in\ccalX_{\mathrm{free}}^{\hat q_\alpha}$ for all $t\in\{0,\ldots,T\}$ and $\hat{\bbx}_T\in\ccalX_{\mathrm{goal}}^{\hat q_\alpha}$. By \eqref{eq:tightened-sets}, this gives $\mathcal B(\hat{\bbx}_t,\hat q_\alpha)\subseteq\ccalX_{\mathrm{free}}$ for all $t$ and $\mathcal B(\hat{\bbx}_T,\hat q_\alpha)\subseteq\ccalX_{\mathrm{goal}}$; in particular $\hat{\bbx}_t\in\ccalX_{\mathrm{free}}$ and $\hat{\bbx}_T\in\ccalX_{\mathrm{goal}}$, so $\bbu_{0:T-1}\in\ccalU(\ccalM)$.} 
Hence, the returned nominal trajectory satisfies the ball-containment conditions in \eqref{eq:alpha-certified}, and the result follows from Prop.~\ref{prop:cp-certification}.
\end{proof}

\section{Experimental Validation}
\label{sec:experiments}
This section evaluates the proposed CP-based planning framework introduced in Section~\ref{sec:cp-cert}. Section~\ref{sec:exp-setup} describes the \textcolor{black}{experimental setup.}
\textcolor{black}{Our evaluation considers both deterministic and stochastic true-system dynamics. Section~\ref{sec:exp-results} validates the guarantee in Proposition \ref{prop:planner-certification}, and characterizes the trade-off between conservativeness and task completion as the prescribed probability increases.}

\subsection{Experimental Setup}
\label{sec:exp-setup}

\textbf{System Dynamics:}
\textcolor{black}{We evaluate our approach on two scenarios representing common forms of model mismatch in practice: known analytical models with incorrect parameters, and learned models for systems whose dynamics are difficult to model analytically \cite{altawaitan2026adapting}.}

(i) First, we consider a deterministic Dubins vehicle. The nominal state is $\hat{\bbx}_t=[\hat x_t,\hat y_t,\hat\theta_t]^\top\in\mathbb{R}^3$ while the true state is $\bbx_t=[x_t,y_t,\theta_t]^\top\in\mathbb{R}^3$. The control input $u_t\in\mathbb{R}$ denotes turn rate. The nominal dynamics are $\hat{\bbx}_{t+1}=\hat{\bbx}_t+\Delta t[\cos(\hat\theta_t),\sin(\hat\theta_t),u_t]^\top$, whereas the true dynamics are $\bbx_{t+1}=\bbx_t+\Delta t[(1-\beta_v)\cos(\theta_t),(1-\beta_v)\sin(\theta_t),u_t]^\top$, where $\Delta t$ is a discretization period. Compared to the nominal model, the true system experiences a translational slowdown, 
\textcolor{black}{as would arise from surface or drivetrain resistance unaccounted for in the nominal model}. The parameter $\beta_v\in \{0.05,0.10,0.15,0.20\}$ controls the magnitude of the model mismatch, yielding four distinct true system models \textcolor{black}{that let us study how the planning behavior varies with mismatch}.

(ii) Our second case study considers a planar quadrotor\textcolor{black}{, a higher-dimensional underactuated system,} with stochastic dynamics. The true state is $\bbx_t=[x_t,y_t,\theta_t,v_{x,t},v_{y,t},\omega_t]^\top\in\mathbb{R}^6$, where $(x_t,y_t)$ denotes the position, $\theta_t$ the pitch angle, $(v_{x,t},v_{y,t})$ the translational velocities, and $\omega_t$ the angular velocity. The control input is $\bbu_t=[\delta T_t,\tau_t]^\top\in\mathbb{R}^2$, where $\delta T_t$ denotes the thrust deviation from hover and $\tau_t$ the body torque. The true dynamics are:
\begin{align}
    x_{t+1} &= x_t+\Delta t\,v_{x,t}, \qquad
    y_{t+1} = y_t+\Delta t\,v_{y,t}, \nonumber\\
    \theta_{t+1} &= \theta_t+\Delta t\,\omega_t, \nonumber\\
    v_{x,t+1} &= v_{x,t}
    +\Delta t\left[-\frac{1}{m}(mg+\delta T_t)\sin\theta_t+w_{x,t}\right], \nonumber\\
    v_{y,t+1} &= v_{y,t}
    +\Delta t\left[\frac{1}{m}(mg+\delta T_t)\cos\theta_t-g+w_{y,t}\right], \nonumber\\
    \omega_{t+1} &= \omega_t+\Delta t\left[\frac{\tau_t}{I}+w_{\omega,t}\right],
    \label{eq:quadrotor-dyn}
\end{align}
where $m=1$, $g=9.8$, and $I=0.25$. The process disturbance is $\bbw_t=[w_{x,t},w_{y,t},w_{\omega,t}]^\top$, where each component is sampled independently from $\mathcal{N}(0,1)$ and clipped to the interval $[-3,3]$. 
The nominal model used by the planner is a learned one-step neural network predictor trained on trajectory data from the noiseless version of \eqref{eq:quadrotor-dyn}. Given the current nominal state $\hat{\bbx}_t$ and control input $\bbu_t$, the network predicts the nominal state increment $\Delta\hat{\bbx}_t=h_\phi(\hat{\bbx}_t,\bbu_t)$, where $h_{\phi}$ is the trained network with parameters $\phi$, and the nominal state is propagated as $\hat{\bbx}_{t+1}=\hat{\bbx}_t+\Delta\hat{\bbx}_t$. \textcolor{black}{Unlike (i), the mismatch in this case is unstructured and arises jointly from approximation errors in the learned dynamics model and stochastic execution noise.}

\textbf{Distribution of Planning Problems:}
In both case studies, planning problems are sampled from the same distribution $\ccalD_{\ccalM}$. \textcolor{black}{
Each planning problem specifies obstacle and goal regions in the robot's 2D workspace. %
These workspace constraints induce the corresponding sets $\ccalX_{\mathrm{free}}$ and $\ccalX_{\mathrm{goal}}$ in the full state space.} Each problem specifies a $25\times25$ environment containing five randomly placed circular obstacles, whose radii are sampled independently from $[1,3]$. The initial state is sampled uniformly from the free space, and the goal region is a circle of radius $5$ centered at a location sampled uniformly from the free space, subject to a minimum distance of $15$ units from the initial position.

\textbf{CP Application:}
\textcolor{black}{A separate calibration dataset containing $N_{\mathrm{cal}}=100$ planning problems is constructed for each true system. Consequently, four calibration datasets are generated for the Dubins car (one for each value of $\beta_v$) and one dataset for the quadrotor. For each calibration planning problem, the NCS is approximated as described in Rem.~\ref{rem:practical-score} and \ref{rem:stochastic-dynamics} using $10$ nominally feasible control sequences generated by a kinodynamic RRT planner \cite{lavalle2001randomized}. 
}

\textbf{Planner Implementation:}
{\textcolor{black}{In all experiments, we solve  $\ccalM_{\alpha}$ using kinodynamic RRT that is run for a fixed number of iterations $N$: we use \textcolor{black}{$N=2,000$ for the Dubins model} and $N=75,000$ for the planar quadrotor.} }

\textbf{Evaluation metrics:}
\textcolor{black}{%
Given a target coverage level $1-\alpha$, we evaluate the proposed framework on $100$ independently sampled planning problems drawn from $\ccalD_{\ccalM}$. We report three metrics.
(1) \emph{Empirical coverage} is computed over the plans returned by the planner. A returned plan is counted as covered if the executed true trajectory remains within distance $\hat q_\alpha$ of the corresponding nominal trajectory at every time step. We then report the fraction of returned plans that satisfy this condition. According to the CP guarantee in \eqref{eq:tube-guarantee}, the empirical coverage is expected to be at least equal to $1-\alpha$.
(2) \emph{Planner success rate} is defined as the fraction of planning problems for which the kinodynamic RRT planner finds a solution to the tightened planning problem $\ccalM_\alpha$ within a fixed iteration budget $N$. If no solution is found, no control sequence is produced, even though the original planning problem $\ccalM$ may still be feasible. As the prescribed probability $1-\alpha$ increases, the radius $\hat q_\alpha$ becomes larger, resulting in a smaller free space $\ccalX_{\mathrm{free}}^{\hat q_\alpha}$ and goal region $\ccalX_{\mathrm{goal}}^{\hat q_\alpha}$ making $\ccalM_\alpha$ harder to solve. \textcolor{black}{Thus, this metric aims to characterize the conservativeness of the proposed approach.}
(3) \emph{Task-completion rate} is evaluated over the subset of planning problems for which the planner returns a solution. It is defined as the fraction of returned plans whose executed true trajectories remain collision-free and reach the goal region, i.e., satisfy \eqref{eq:fixed-alpha-goal}. 
\textcolor{black}{Due to Proposition~\ref{prop:planner-certification}, the task completion rate is expected to be at least $1-\alpha$}. 
}

\textbf{Nominal Planning Baseline:} To demonstrate the benefits of accounting for nominal-model imperfections, we compare the proposed framework \textcolor{black}{against RRT applied directly to the nominal model.}
\textcolor{black}{Specifically, the baseline} uses the same nominal model and iteration budget as our approach, but plans 
\textcolor{black}{over the original planning problem, $\ccalM$.}
To ensure a fair comparison and isolate the effect of conformal tightening, the baseline is evaluated on the same subset of planning problems for which our method returns a solution. %

\begin{figure}
    \centering
    \includegraphics[width=1.0\linewidth]{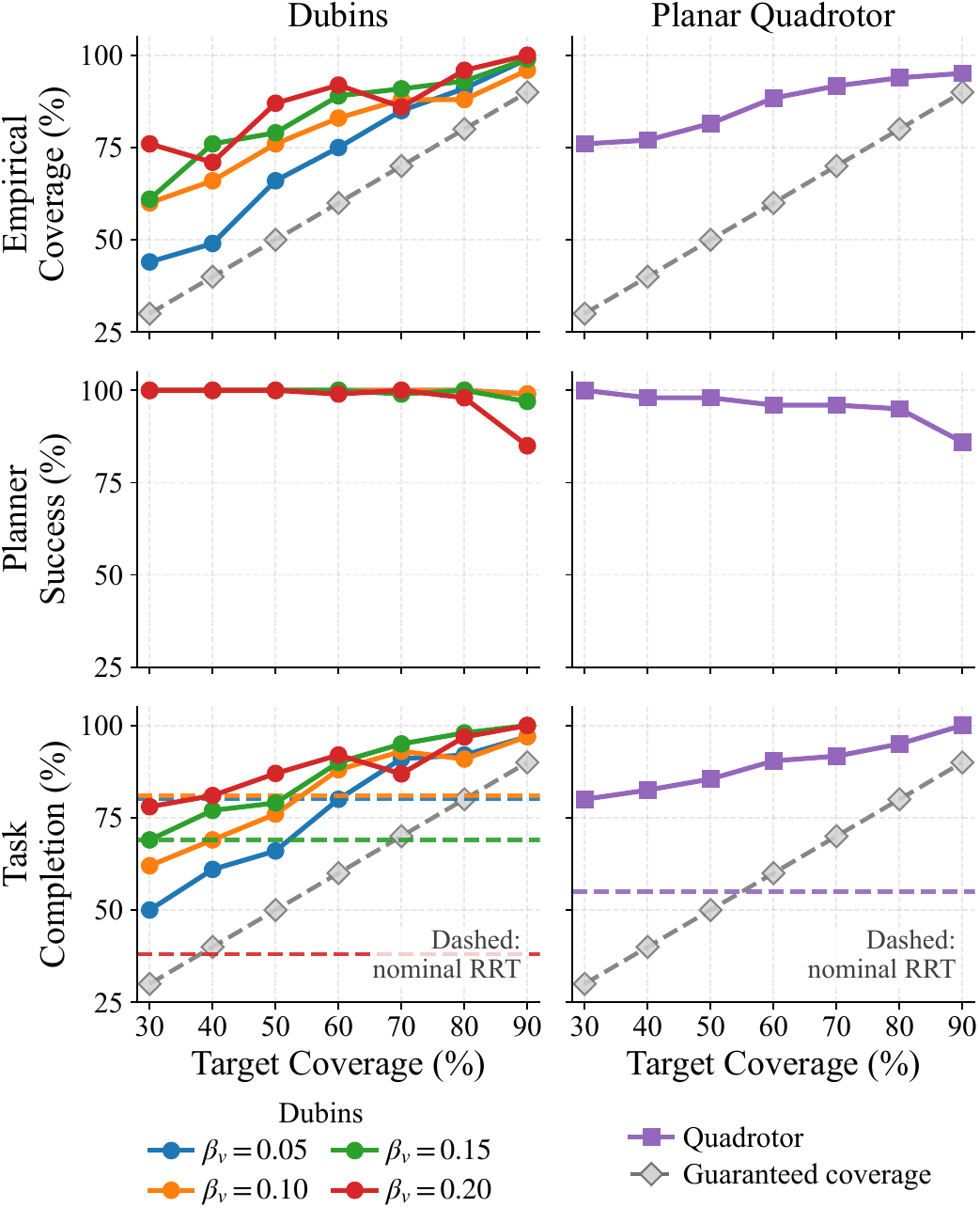}\vspace{-0.2cm}
    \caption{Empirical results for the Dubins vehicle (left column) and planar quadrotor (right column). Rows report empirical coverage, planner success rate, and task-completion rate. Solid curves correspond to RRT applied to the tightened problem $\ccalM_\alpha$. In the Dubins column, colors indicate the model-mismatch level $\beta_v$. The gray dashed diagonals denote the \textcolor{black}{prescribed probability $1-\alpha$}. In the bottom row, dashed horizontal lines denote the nominal RRT baseline in the corresponding color.}\vspace{-0.5cm}
    \label{fig:results}
\end{figure}

\subsection{Empirical Results}
\label{sec:exp-results}

\textcolor{black}{To characterize the effect of the prescribed coverage level on planning performance, we evaluate the proposed framework for $1-\alpha\in\{0.30,0.40,\ldots,0.90\}$. 
}

\noindent\textbf{Dubins vehicle:}
\textcolor{black}{The left column of Figure~\ref{fig:results} summarizes the results for the Dubins vehicle case study. The top-left panel compares the \textit{empirical coverage} against the prescribed level $1-\alpha$. In all cases, the empirical coverage exceeds the target level, as expected due to \eqref{eq:tube-guarantee}. As the model mismatch increases (larger $\beta_v$), the gap between the empirical and prescribed coverage also increases. This behavior is expected because larger mismatch produces larger nominal-to-true trajectory deviations, resulting in larger calibration scores and, consequently, a larger conformal threshold $\hat q_\alpha$. As a result, the conformal balls become more conservative, leading to empirical coverage that exceeds the prescribed level.}

\textcolor{black}{This conservativeness has a direct effect on \textit{planner success rate}. Specifically, increasing either $1-\alpha$ or the model mismatch enlarges $\hat q_\alpha$, which shrinks the tightened free and goal sets. Thus, fewer nominal trajectories satisfy the conditions in Proposition~\ref{prop:cp-certification}, reducing the planner success rate. The middle-left panel shows this trade-off. However, the degradation remains moderate over the sampled  problems. %
}

\textcolor{black}{The \textcolor{black}{bottom-left panel} reports the \emph{task-completion rate}. As predicted by Proposition~\ref{prop:planner-certification}, the observed task-completion rate consistently exceeds $1-\alpha$. %
Moreover, the proposed framework consistently outperforms the nominal-planning baseline, with the performance gap increasing as both the model mismatch and $1-\alpha$ increase. Overall, these results highlight the main trade-off of the proposed approach: increasing the prescribed coverage level $1-\alpha$ yields more conservative planning, reducing planner success rate while substantially improving the reliability of the returned plans.}

\noindent\textbf{Planar quadrotor:} \textcolor{black}{The right column of Fig.~\ref{fig:results} summarizes the results for the stochastic quadrotor case. Similar trends are observed. The empirical coverage consistently exceeds the prescribed level $1-\alpha$, validating the CP guarantee despite the combination of learned-model error and stochastic execution noise. As the prescribed coverage level increases, the planner success rate decreases because the larger conformal threshold $\hat q_\alpha$ yields a more conservative tightening of the free and goal sets. Nevertheless, the success rate remains high across the tested range. At the same time, the task-completion rate increases with the prescribed coverage level and consistently outperforms the nominal-planning baseline. Specifically, the task-completion rate of our method rises from about $80\%$ at $1-\alpha=30\%$ to $100\%$ at $1-\alpha=90\%$, while the nominal RRT baseline achieves about $55\%$.
}

\section{Conclusion}
\label{sec:conclusion} 

\textcolor{black}{This paper addressed motion planning for systems with unknown dynamics using an approximate nominal model. 
Using CP, we derived a probabilistic bound on the nominal-to-true trajectory deviation. We used this bound to construct tightened planning problems such that any control sequence solving them under the nominal model also solves the original planning problem on the true system with the prescribed probability. Comparative experiments on unseen planning problems validated the theoretical guarantees. Future work will focus on reducing conservativeness and developing adaptive-risk planning strategies when no plan can be computed with the desired probability.
}

\bibliographystyle{IEEEtran}
\bibliography{references}

\end{document}